\documentclass[letterpaper, 10 pt, conference]{ieeeconf}  

\IEEEoverridecommandlockouts                              

\usepackage{graphics}
\usepackage{amsmath}
\usepackage{amssymb}
\usepackage{mathtools}
\usepackage{xcolor}

\usepackage{url}
\usepackage{bbm}
\usepackage{pifont}
\usepackage{subcaption}
\usepackage{algorithm}
\usepackage{algorithmic}

\newtheorem{theorem}{Theorem}

\newtheorem{lemma}{Lemma}
\newtheorem{corollary}{Corollary}
\newtheorem{definition}{Definition}
\newtheorem{problem}{Problem}

\newtheorem{example}{Example}

\title{\LARGE \bf
Model-Free Reinforcement Learning for Stochastic Games\\ with Linear Temporal Logic Objectives
}

\author{Alper Kamil Bozkurt, Yu Wang, Michael Zavlanos, and Miroslav Pajic
\thanks{This work is sponsored in part by the ONR under agreements N00014-17-1-2504, N00014-20-1-2745 and N00014-18-1-2374, AFOSR award number FA9550-19-1-0169, and the NSF CNS-1652544 and CNS-1932011 grants.}
\thanks{Alper Kamil Bozkurt, Yu Wang, Michael Zavlanos, and Miroslav~Pajic are with Duke University, Durham, NC 27708, USA, {\tt\small \{alper.bozkurt, yu.wang094, michael.zavlanos, miroslav.pajic\}@duke.edu}}}

\begin{document}

\maketitle
\thispagestyle{empty}
\pagestyle{empty}

\begin{abstract}
We study the problem of synthesizing control strategies for Linear Temporal Logic (LTL) objectives in unknown environments. We model this problem as a turn-based zero-sum stochastic game between the controller and the environment, where the transition probabilities and the model topology are fully unknown. The winning condition for the controller in this game is the satisfaction of the given~LTL specification, which can be captured by the acceptance condition of a deterministic Rabin automaton (DRA) directly derived from the LTL specification. We introduce a model-free reinforcement learning (RL) methodology to find a strategy that maximizes the probability of satisfying a given LTL specification when the Rabin condition of the derived DRA has a single accepting pair. We then generalize this approach to LTL formulas for which the Rabin condition has a larger number of accepting pairs, providing a lower bound on the satisfaction probability.
Finally, we illustrate applicability of our RL method on two motion planning case studies.
\end{abstract}

\section{Introduction}

Reinforcement learning (RL) provides 
methods for finding solutions to sequential decision-making problems where the objective is to maximize the discounted rewards 
in unknown environments~\cite{sutton2018}. 
Recent RL developments, such as the use of deep learning and Monte Carlo tree search, have led to successful applications in a range of domains including (super) human-level board and Atari game playing~\cite{mnih2015,silver2016}. 
Yet, despite these advances, there is a need to provide RL methods with robustness and safety guarantees to allow their use in real-world systems, particularly in cyber-physical~\cite{dreossi_CompositionalFalsificationCyberPhysical_2017,sun_FormalVerificationNeural_2019,tran_SafetyVerificationCyberPhysical_2019,yaghoubi_GrayboxAdversarialTesting_2019,zarei_StatisticalVerificationLearningbased_2020} and robotic systems~\cite{cheng_EndtoEndSafeReinforcement_2019,taylor_LearningSafetyCriticalControl_2020,wang_HyperpropertiesRoboticsMotion_2020}.
To achieve this, a major challenge is to design reward signals such that a strategy optimizing the discounted reward achieves the desired task~\cite{barto2019}.

Linear temporal logic (LTL) allows formal capturing desired temporal properties of a control task or system (e.g., robot planning). Using LTL to specify the task objective can prevent unintended consequences of an optimal strategy for a shaped reward function. 
LTL specifications can be directly extracted from high-level requirements in robot planning and control~\cite{kress2008translating,guo2013revising,wongpiromsarn2012receding}. Thus,
synthesizing controllers from LTL objectives for Markov decision processes (MDPs) using RL has attracted significant attention~\cite{bozkurt2019,hasanbeig2019,hahn2019}. 
These methods generally translate the LTL specification into a limit-deterministic B\"uchi automaton (LDBA), which is then composed with the initial MDP, and design a reward function based on the acceptance condition of the automaton. Augmentation of the state space using the states of the LDBA solves the memory requirements of the task. In addition, the B\"uchi acceptance condition, which is repeated reachability, enables the use of simple reward functions.

Such LDBA-based rewarding approaches are not well-suited for stochastic games because LDBAs and many other nondeterministic automata, in general, cannot be used in solving games \cite{henzinger2006}. 
Hence, there are few studies on learning-based synthesis from temporal objectives for stochastic games. One approach 
is to translate the LTL specifications to deterministic automata with more complicated acceptance conditions than LDBAs.
For example, \cite{wen2016} proposed a probably approximately correct (PAC) learning algorithm for stochastic games with LTL and discounted sums of rewards objective. However, the approach assumes that the transition graph (i.e., topology) is known a-priori, the LTL objective must belong to a very limited subset of LTL formulas that can be translated into a deterministic B\"uchi automaton (DBA), and there exists a strategy that almost surely satisfies the LTL objective. These assumptions allow the pre-computation of the winning regions before the learning.

Recently, \cite{ashok2019} introduced a model-based learning method that yields PAC guarantees for reachability objectives. The method uses on-the-fly detection of (simple) end components of stochastic games, and careful construction of the confidence intervals on the transition probabilities. However,  only a small fragment of LTL formulas can be expressed by the reachability objectives. Also, as a model-based method, it is not efficient in terms of space requirements when the number of possible successors of actions is~not~small.

To address these limitations, in this work we introduce a model-free RL approach to synthesize controllers for stochastic games, such that the obtained control policies maximize the (worst-case) probabilities of satisfying the given LTL task objectives. To achieve this, we translate the  LTL specification into a DRA and introduce a reward and a discount (or termination) function based on the Rabin acceptance condition. We first consider DRAs with a single accepting pair and prove any model-free RL algorithm using these functions converges a desired strategy for a sufficiently large discount factor. We then generalize our method to any LTL specification, for which the DRA may have an arbitrary number of accepting pairs; for such specifications, we establish a lower bound on the satisfaction probability. Lastly, we show the applicability of our RL approach on two robot motion planning case studies. 

\section{Preliminaries and Problem Statement}\label{section:prelim}

\subsection{Stochastic (Turn-Based) Two-Player Games} \label{section:sg}

We use turn-based stochastic games to model the interaction between the controller (i.e., Player~1) and unpredictable environment (i.e., Player~2), where actions have probabilistic outcomes. The controller 
can only choose actions in certain states; the rest of the states are in control of the environment. 

\begin{definition}[Stochastic Games] \label{def:sg}
A (labeled turn-based) two-player stochastic game is a tuple $\mathcal{G}=(S,(S_\mu,S_\nu), \allowbreak A, P, s_0, \textnormal{AP}, L)$, where
$S$ is a finite set of states;
$S_\mu \subseteq S$ is the set of states where the controller chooses actions; 
$S_\nu = S \setminus S_\mu $ is the set of states at which the environment chooses actions;
$A$ is a finite set of actions whereas $A(s)$ denotes the set of actions that can be taken in state $s\in S$;
$P: S \times A \times S \to [0,1]$ is the transition probability function such that for all $s \in S$, $\sum_{s' \in S} P(s, a, s')=1$ if $a\in A(s)$, and $0$ otherwise;
\textnormal{AP} is a finite set of atomic propositions;  and 
$L\hspace{-2pt}:\hspace{-2pt}S\to 2^{\textnormal{AP}}$ is a labeling~function.
\end{definition}

A \emph{path} in a stochastic game $\mathcal{G}$ is an infinite sequence of states $\sigma=s_0s_1,\dots$ such that for all $t\geq 0$, there exists an action $a \in A(s_t)$ where $P(s_t,a,s_{t+1})>0$. We write $\sigma[t]$, $\sigma[{:}t]$ and $\sigma[t{+}1{:}]$ to denote the state $s_t$, the prefix $s_0 s_1\dots s_{t}$ and the suffix $s_{t+1}s_{t+2}\dots$ of the path,  respectively.

The behavior of the players in stochastic games is described by strategies, which maps the previously visited states to the available actions in the current state.

\begin{definition}[Strategies]
For a game $\mathcal{G}$, let $S_\mu^+$($S_\nu^+$) denote the set of all \textbf{finite} prefixes $\sigma_{\text{f}}^{s_\mu}$($\sigma_{\text{f}}^{s_\nu}$) ending with a state $s_\mu \in S_\mu$($s_\nu \in S_\nu$, respectively) of paths in the game. Then, 
\begin{itemize}

\item a (pure) \textbf{control strategy} $\mu$ 
is a function $\mu:S_\mu^+\to A$ such that $\mu(\sigma_{\text{f}}^{s_\mu})\in A(s_\mu)$ for all $\sigma_{\text{f}}^{s_\mu} \in S_\mu^+$,

\item a (pure) \textbf{environment strategy} $\nu$ 
is a function $\nu:S_\nu^+\to A$ such that $\nu(\sigma_{\text{f}}^{s_\nu})\in A(s_\nu)$ for all $\sigma_{\text{f}}^{s_\nu} \in S_\nu^+$,

\item a strategy $\pi$ is \textbf{memoryless}, if it only depends on the current state, i.e., $\pi(\sigma_f^s) = 
\pi(\sigma_f^{s'})$ if $s=s'$ for any $\sigma_f^s$ and $\sigma_f^{s'}$, and thus can be defined as $\pi: S \to A$.
\end{itemize}
\end{definition}

The induced Markov chain (MC) of a game $\mathcal{G}$ under a strategy pair $(\mu,\nu)$ is tuple $\mathcal{G}_{\mu,\nu}=(S,P_{\mu,\nu},s_0,\textnormal{AP},L)$,
where 
\begin{align*}
    P_{\mu,\nu}(s, s') = \begin{cases}
    P(s, \mu(s), s') & \textrm{if } s \in S_\mu \\
    P(s, \nu(s), s') & \textrm{if } s \in S_\nu \\
    \end{cases}.
\end{align*}
We denote by $\mathcal{G}_{\mu,\nu}^{s}$ the MC resulting from changing the initial state from $s_0$ to $s\in S$ in $\mathcal{G}_{\mu,\nu}$, and use $\sigma \sim \mathcal{G}_{\mu,\nu}^{s}$ to denote a random path sampled from $\mathcal{G}_{\mu,\nu}^{s}$. Finally, a \emph{\textbf{bottom strongly connected component}} (BSCC) of the (induced) MC $\mathcal{G}_{\mu,\nu}$ is a strongly connected component with no outgoing transitions; we use $\mathcal{B}(\mathcal{G}_{\mu,\nu})$ to denote the set of all BSCCs of $\mathcal{G}_{\mu,\nu}$.

\subsection{LTL and Deterministic Rabin Automata} \label{section:ltl_prelim}

We capture the desired behaviors of a labeled stochastic game by LTL specifications, which impose requirements on the label sequences corresponding to the infinite paths of the game. LTL offers a formal language that can be used to specify desired temporal characteristics or tasks of a controller~\cite{baier2008}. In addition to the standard Boolean operators: negation ($\neg$) and conjunction ($\wedge$), LTL formulas can include two temporal operators, namely next ($\bigcirc$) and until ($\textsf{U}$), and any recursive combinations of the operators. The formal syntax of LTL is defined as~\cite{baier2008}
\begin{align}
    \hspace{-0.5em} \varphi \coloneqq  \mathrm{true} \mid a \mid \varphi_1 \wedge \varphi_2 \mid \neg \varphi \mid \bigcirc \varphi \mid \varphi_1 \textsf{U} \varphi_2, ~ {a\in\textnormal{AP}}, \label{eq:ltl}
\end{align}
where AP is a set of atomic propositions.

For a stochastic game $\mathcal{G}$ with a labeling function $L$, 
the LTL semantics is defined over the paths of the game. A path $\sigma$ satisfies an LTL formula $\varphi$, denoted by $\sigma \models \varphi$ if:
\begin{itemize}
    \item $\varphi=a$ and $a \in L(\sigma[0])$,
    \item $\varphi=\varphi_1 \wedge \varphi_2$, $\sigma \models \varphi_1$ and $\sigma \models \varphi_1$,
    \item $\varphi=\neg \varphi'$ and $\sigma \not\models \varphi'$,
    \item $\varphi=\bigcirc \varphi'$ and $\sigma[1{:}] \models \varphi'$,
    \item $\sigma \models \varphi_1 \textsf{U} \varphi_2$, $\exists i. \sigma[i] \models \varphi_2$ and $ \forall j<i. \sigma[j] \models \varphi_1$.
\end{itemize}
Other operators can be easily derived: $\varphi_1 \lor \varphi_2 \coloneqq \neg (\neg \varphi_1 \land \neg \varphi_2)$, $\varphi_1 \to \varphi_2 \coloneqq \neg \varphi_1 \lor \varphi_2$,
(eventually) $\lozenge \varphi \coloneqq \mathrm{true}\ \textsf{U}\ \varphi$; and (always) $\square \varphi \coloneqq \neg (\lozenge \neg \varphi)$~\cite{baier2008}.

Any LTL formula can be systematically transformed into a DRA that accepts the language of all paths satisfying the formula~\cite{baier2008}. DRAs are similar to deterministic finite automata except for the acceptance criteria, which is defined based on infinite visits of some states.

\begin{definition}[Deterministic Rabin Automata] \label{def:dra}
A DRA is a tuple $\mathcal{A}=(Q,\Sigma, \delta, q_0, \textnormal{Acc})$ where
$Q$ is a finite set of states;
$\Sigma$ is a finite alphabet;
$\delta: Q \times \Sigma \to Q$ is the transition function;
$q_0 \in Q$ is an initial state; and
$\textnormal{Acc}$ is a set of $k$ accepting pairs $\{(C_i, B_i)\}_{i=1}^k$ such that $C_i, B_i \subseteq Q$. 
\end{definition}

An infinite path $\sigma$ is accepted by the DRA if it satisfies the \emph{\textbf{Rabin condition}} -- i.e., there exists a pair $(C_i,B_i) \in \textnormal{Acc}$  such that the states in $C_i$ are visited finitely many times and at least one state in $B_i$ visited is infinitely often, namely,
\begin{align}
    \exists i: \ inf(\sigma)\cap C_i =\varnothing \ \wedge \ inf(\sigma) \cap B_i \neq \varnothing, \label{eq:dra_acc}
\end{align}
where $inf(\sigma)$ denotes the set of states visited by the path $\sigma$ for infinitely many times.
The \emph{\textbf{Rabin index}} of an LTL formula is the minimum number of accepting pairs a DRA recognizing the formula can have. Without loss of generality, we assume that the number of accepting pairs, $k$, is equal to the Rabin index of the accepted language.

\begin{example} 
Figure~\ref{fig:dra} illustrates a DRA derived from the LTL formula $\varphi=\square \lozenge b \vee \lozenge \square d$, with Rabin acceptance sets $B_1{=}\{q_1\}$, $C_2{=}\{q_0\}$ and $B_2{=}\{q_2\}$ (i.e., acceptance condition $\textnormal{Acc}{=}\{(\varnothing,B_1),(C_2,B_2)\}$). Note that consuming a label having $b$ in it, leads to a transition to the state $q_1$ from any state. Thus, any path $\sigma$ containing infinitely many states labeled with $b$ induces an execution that visits $q_1$ infinitely many times; thereby satisfying the Rabin condition. Other executions satisfying the Rabin conditions are the ones visiting $q_2$ infinitely many times but $q_0$ only finitely many times. Those can be only produced by the paths that after some~point, do not contain a state whose label is not $d$.
\end{example}

\begin{figure}[!t]
    \centering
    \includegraphics[width=0.4\textwidth]{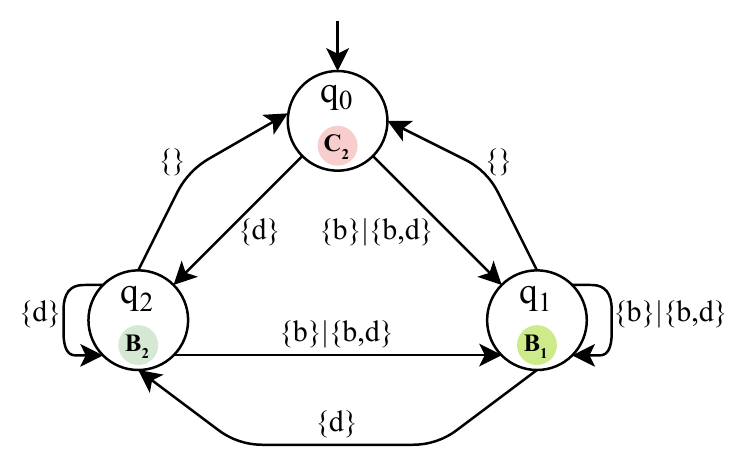}
    \caption{A DRA derived from the LTL formula $\varphi{=}\square \lozenge b {\vee} \lozenge \square d$. Here, 
    $B_1{=}\{q_1\}$, $C_2{=}\{q_0\}$ and $B_2{=}\{q_2\}$ are the states in the Rabin acceptance condition $\textnormal{Acc}{=}\{(\varnothing,B_1),(C_2,B_2)\}$.}
    \label{fig:dra}
\end{figure}

\subsection{Reinforcement Learning for Stochastic Games} \label{section:rl}

Let $R: S\to \mathbb{R}$ be a \emph{reward function} and $\gamma \in (0, 1)$ be the discount factor for a given two-player zero-sum stochastic game $\mathcal{G}$. 
The \emph{value} of a state $s$ under the strategy pair $(\mu,\nu)$ is the expected sum of the discounted reward
\begin{align}
    v_{\mu,\nu}(s) = \mathbb{E}_{\sigma \sim \mathcal{G}_{\mu,\nu}}\left[ \sum_{i=0}^{\infty} \gamma^{i}R(\sigma[t{+}i])
    \ \bigg\vert \ \sigma[t]=s \right],
\end{align}
for any fixed $t \in \mathbb{N}$, such that $Pr_{\sigma \sim \mathcal{G}_{\mu,\nu}}\left[\sigma[t]{=}s \right] > 0$. We omit the subscript $_{\sigma \sim \mathcal{G}_{\mu,\nu}}$ from the expectation and write $\mathbb{E}$ rather than $\mathbb{E}_{\sigma \sim \mathcal{G}_{\mu,\nu}}$.

The RL objective is to find an optimal \emph{control strategy} $\mu_*$ to maximize the values of every state 
under the worst \emph{environment strategy}. A pure and memoryless optimal strategy always exists in two-player turn-based zero-sum stochastic games \cite{shapley1953stochastic,littman1994}. The optimal values in these games satisfy~\cite{hu2003nash} 
\begin{align}
v_*(s) = \max_\mu \min_\nu v_{\mu,\nu}(s),
\end{align}
where $\mu$ and $\nu$ are pure and memoryless control and environment strategies, respectively. 
In addition, the \emph{optimal values} $v_*(s)$ need to satisfy the Bellman equations
\begin{align*}
    v_*(s) = R(s) {+} \gamma\begin{cases}
        \max\limits_{a \in A(s)} \sum\limits_{s'\in S}P(s,a,s')v_*(s') & \textrm{if } s \in S_\mu, \\
        \min\limits_{a \in A(s)} \sum\limits_{s'\in S}P(s,a,s')v_*(s') & \textrm{if } s \in S_\nu.
    \end{cases}
\end{align*}

Model-free RL methods aim to learn the optimal values of the stochastic game, in which neither the transition probabilities nor the game topology are known, without explicitly constructing a transition model of the game. A popular example is the minimax-Q method that generalizes the standard off-policy Q-learning algorithm to stochastic games. The minimax-Q method can learn the optimal values from any (likely non-optimal) strategies used during learning as long as all the actions in each state are chosen infinitely often~\cite{littman1994,bowling2000}.

\subsection{Problem Formulation} \label{section:problem}

We assume that the considered game $\mathcal{G}$ is fully observable for both players; i.e., 
both the controller and  environment are aware of the current game state.
A control synthesis problem can be roughly described as finding a strategy for the controller in a stochastic game such that all the paths produced under the strategy satisfy the given LTL specification regardless of the behavior of the environment. If such a strategy does not exist, the objective becomes to find a strategy that maximizes the probability that a produced path satisfies the specification in the worst~case.

To simplify our notation, we use $Pr_{\mu,\nu}^{\mathcal{G}}(s \models \varphi)$ to denote the probability of the paths starting from the state $s$ that satisfy the formula $\varphi$ under~the~strategy pair $({\mu,\nu})$ -- i.e., 
\begin{align}
    Pr_{\mu,\nu}^\mathcal{G}(s \models \varphi) \coloneqq Pr_{\sigma \sim \mathcal{G}_{\mu,\nu}^s}( \sigma \models \varphi );
\end{align}
we write $Pr_{\mu,\nu}(\mathcal{G} \models \varphi)$ for $Pr_{\mu,\nu}^{\mathcal{G}}(s_0 \models \varphi)$ and use $Pr_*$ to denote the maximin probability $\max_{\mu} \min_{\nu} Pr_{\mu,\nu}$. We can now formally define the considered problem as follows.

\begin{problem} \label{problem}
Given a labeled turn-based stochastic game~$\mathcal{G}$, where the transition probabilities are fully unknown, and an LTL specification $\varphi$, design a model-free RL algorithm that finds a \emph{pure finite-memory} controller strategy $\mu_*$ such~that
\begin{align}
    Pr_{\mu_*,\nu}\left(\mathcal{G} \models \varphi \right) \geq 
    Pr_*\left(\mathcal{G} \models \varphi \right) \label{eq:problem_statement}
\end{align}
for any environment strategy $\nu$.
\end{problem}


\section{Learning for Stochastic Rabin Games} \label{section:method}

In this section, we describe our model-free RL approach to derive control strategies that maximize the (worst-case) probability of satisfying a given LTL formula. First, we describe the 
product game construction, a key step in reducing the problem of satisfying an LTL specification into the problem of satisfying a Rabin condition. We then consider the case where the DRA derived from the LTL objective~$\varphi$ has a single Rabin pair, 
and  introduce our 
rewarding and discounting mechanisms 
based on it.
We show that maximization of the discounted reward maximizes the minimal probability of satisfying the single pair Rabin condition, and thus the initial LTL objective. Finally, we provide a generalization to Rabin conditions with an arbitrary ($k{>}1$) number of accepting pairs; thereby allowing the use of or method for all possible LTL specifications. Specifically, we  construct a game from $k$ copies of the original game, where only a single Rabin pair needs to be satisfied;
an optimal solution to this game guarantees a lower bound on the satisfaction probabilities.

\subsection{Product Game Construction} \label{section:product}
Our main idea is that by forming an augmented state~space, Problem~\ref{problem} can be reduced into finding a memoryless control strategy. 
Specifically, we compose the states of the game $\mathcal{G}$ with the states of the DRA $\mathcal{A}$ derived from the LTL specification $\varphi$. Then, the goal in this space is to satisfy the Rabin acceptance condition, for which memoryless control strategies suffice~\cite{chatterjee2012}. 

\begin{definition}[Product Game] \label{def:product} A~product~game $\mathcal{G}^\times = (S^\times, (S_\mu^\times,S_\nu^\times), A^\times, P^\times, s_0^\times, \textnormal{Acc}^\times)$ of a labeled turn-based stochastic game $\mathcal{G}{=}(S, (S_\mu,S_\nu), A, P, s_0, \textnormal{AP}, L)$ and a DRA $\mathcal{A}=(Q,2^\textnormal{AP}, \delta, q_0, B)$  is defined as follows: 
\begin{itemize}
    \item $S^\times=S\times Q$ is the set of augmented states, where the initial state $s_0^\times$ is $\langle s_0,q_0 \rangle$,
    \item $S_\mu^\times=S_\mu\times Q$ is the set of augmented controller states,
    \item $S_\nu^\times=S_\nu\times Q$ is the set of augmented environment states,
    \item $A^\times=A$
is the set of actions,
    \item $P^\times:S^\times \times A^\times \times S^\times \to [0,1]$ is the transition function:
    \begin{align}
        &P^\times(\langle s,q \rangle,a,\langle s',q'\rangle) {=} 
        \begin{cases}
            P(s,a,s') & \textrm{if } q' {=} \delta(q,L(s)) \\
            0 & \textnormal{otherwise}
        \end{cases} \notag 
    \end{align}

    \item $\textnormal{Acc}^\times$ is a set of $k$ accepting pairs $\{(C_i^\times, B_i^\times)\}_{i=1}^k$ where $C_i^\times = C_i \times Q$ and $B_i^\times = B_i \times Q$.
\end{itemize}
\end{definition}

Similarly to (\ref{eq:dra_acc}), a path $\sigma^\times$ of the product game $\mathcal{G}^\times$ satisfies the Rabin condition if there exists $i$, such that
$
    \ inf(\sigma^\times)\cap C_i^\times =\varnothing \ \wedge \ inf(\sigma^\times) \cap B_i^\times \neq \varnothing. \label{eq:prod_acc}
$
Finally, we refer to a product game with $k$ accepting pairs as a Rabin($k$)~game.

There is a one-to-one correspondence between the paths in the product game and the original game. Similarly, a strategy for the product game induces a strategy in the original game and vice versa. Note that, however, the corresponding strategies in the original game require additional memory described by the DRA -- i.e., the strategy in the original game may not be memoryless. On the other hand, the probability of satisfying the Rabin condition under any strategy pair in the product game is equivalent to the probability of satisfying the LTL formula in the original game under the corresponding strategy pair.
Hence, in the rest of the section we focus on the product games, i.e., stochastic Rabin games; to simplify our notation, we omit the superscript $ ^\times$ and use $\mathcal{G}=(S, (S_\mu, \allowbreak S_\nu), A, P, s_0, \textnormal{Acc})$ and $s\in S$ instead of $\mathcal{G}^\times=(S^\times, \allowbreak (S_\mu^\times,S_\nu^\times) A^\times,P^\times s_0^\times, \textnormal{Acc}^\times)$ and $\langle s,q \rangle \in S^\times$.

\subsection{Rabin($1$) Condition to Discounted Rewards} 
\label{section:rabin_1}
We first consider the case where the LTL formula $\varphi$ has one accepting pair in the Rabin acceptance condition.
In stochastic Rabin($1$) games, where $\textnormal{Acc} = \{(C,B)\}$, the objective of the controller is to repeatedly visit some states in $B$ and visit the states in $C$ for only finitely many times. 
On the other hand, the environment's goal is to prevent this from happening, which can be also expressed as a Rabin condition with two accepting pairs $\textnormal{Acc}' = \{(\varnothing,C),(B,S)\}$. 
Thus, pure and memoryless strategies  suffice for both players on the considered product game~\cite{chatterjee2012}.

To find a control strategy that satisfies~\eqref{eq:problem_statement} for stochastic Rabin($1$) games using RL, our key idea is to assign small rewards to the states in $B$ to encourage visiting $B$ states as often as possible; but discount more compared to the other states to eliminate the importance of the frequency of visits. In addition, we discount even more in the states in $C$ without giving any rewards, which diminishes the worth of the rewards to be obtained by visiting the states in $B$.
The following proposition summarizes our 
key results.

\begin{theorem} \label{theorem:rabin}
Consider a given turn-based stochastic Rabin($1$) product game $\mathcal{G}$ and the return of any path $\sigma$ defined~as
\begin{align}
    \vspace{-1em} G_t(\sigma) &\coloneqq \sum\nolimits_{i=0}^{\infty} R_B(\sigma[t{+}i])\cdot \prod\nolimits_{j=0}^{i-1} \Gamma_{B,C}(\sigma[t{+}j]), \label{eq:return_updated}
\end{align} 
where $\prod_{j=0}^{-1} \coloneqq 1$, $R_B:S\to[0,1)$ and $\Gamma_{B,C}: S \to (0,1)$ are the reward and the terminal functions defined as
\begin{align}
    R_B(s) &\coloneqq \begin{cases}
        1-\gamma_B & \textrm{if }s \in B \\
        0 & \textrm{if }s \notin B
    \end{cases}, \\
        \Gamma_{B,C}(s) &\coloneqq \begin{cases}
        \gamma_B & \textrm{if }s \in B \\
        \gamma_C & \textrm{if }s \in C \\
        \gamma & \textrm{if }s \in S \setminus (B \cup C)
    \end{cases}.  \label{eq:reward_discount}
\end{align}
Here, $\gamma_B$ and $\gamma_C$ are functions of $\gamma$ such that $0<\gamma_C(\gamma)<\gamma_B(\gamma)<\gamma<1$, and $\lim\limits_{\gamma \to 1^-} \gamma_B = \lim\limits_{\gamma \to 1^-} \gamma_C = 1 $, as well as
\begin{align} \label{eq:gamma_lim}
    \lim_{\gamma \to 1^-} \frac{1 - \gamma}{1 - \gamma_B (\gamma)} = \lim_{\gamma \to 1^-} \frac{1 - \gamma_B(\gamma)}{1 - \gamma_C (\gamma)} = 0. 
\end{align}
Then, the value of the game $v_{\mu,\nu}^\gamma$ (i.e., the expected return $\mathbb{E}\left[G_t(\sigma)\right]$) for the strategy pair~$(\mu,\nu)$ and the discount factor $\gamma$ satisfies that for all states $s \in S$ it holds that
\begin{align} \label{eq:thm2}
    \lim_{\gamma \to 1^{{-}}} v_{\mu,\nu}^\gamma(s) = Pr^{\mathcal{G}}_{\mu,\nu}(s \models \varphi_{B,C});
\end{align} 
here, $\varphi_{B,C} \coloneqq \square \lozenge B \wedge \neg \square \lozenge C$ is the Rabin condition of the DRA derived from the LTL objective $\varphi$.
\end{theorem}

Before proving Theorem~\ref{theorem:rabin}, we show 
Lemma~\ref{lemma:inequalities}, later used to establish bounds on the states' values.
Intuitively, Lemma~\ref{lemma:inequalities} shows that if we replace a state on a path with a state in $B$, we obtain a larger or equal return; if we replace it with a state in $S\setminus B$, we obtain a smaller or equal return; and the return is always between 0 and~1.

\begin{lemma} \label{lemma:inequalities}
For any path $\sigma$ and a fixed $t\geq0$, in a stochastic game with the path return defined as in~\eqref{eq:return_updated}, it holds that

\vspace{-10pt}
\small
\begin{align}
\label{eq:l1_in1}
    \hspace{-0.5em} \gamma_C G_{t+1}(\sigma) \leq \gamma G_{t+1}(\sigma) &\leq G_t(\sigma) \leq 1{-}\gamma_B {+} \gamma_B G_{t+1}(\sigma),  \\
    \hspace{-0.5em}0 &\leq G_t(\sigma) \leq 1. \label{eq:l1_in2}
\end{align}
\end{lemma}
\normalsize

\begin{proof}
It holds that $0 \leq \gamma_C G_{t+1}(\sigma) \leq \gamma G_{t+1}(\sigma)$  since the rewards are nonnegative and $\gamma > \gamma_C$.
Now, let us assume that we do not discount in the states that do not belong to $B$, i.e., we replace $\gamma$ and $\gamma_C$ with $1$ in $G_t(\sigma)$ as
\[
\Gamma_B'(s) = \begin{cases} \gamma_B & \textrm{if }s \in B \\ 1 & \textrm{if }s \notin B \end{cases}.
\]
Then the return
\[
G'_t(\sigma) = \sum_{i=0}^{\infty} R_B(\sigma[t{+}i]) \prod_{j=0}^{i-1} \Gamma_B' (\sigma[t{+}j])
\]
is evidently larger than $G_t(\sigma)$. Furthermore, it holds that $G_t(\sigma) \leq G'_t(\sigma) \leq 1 - \gamma_B^(b+1) \leq 1$,
where $b$ is the number of times a $B$ state is visited; thus, the return of any path is bounded by 1 from above.

From the path return definition~\eqref{eq:return_updated},  it holds that
\begin{align}
    G_t(\sigma) = \begin{cases}
        1{-}\gamma_B + \gamma_B G_{t+1}(\sigma) & \textrm{if }\sigma[t] \in B \\
        \gamma_C G_{t+1}(\sigma) & \textrm{if }\sigma[t] \in C \\
        \gamma G_{t+1}(\sigma) & \text{otherwise}
    \end{cases}. \label{eq:return_recursive}
\end{align}
Using $\gamma_B < \gamma$ and $G_{t+1}(\sigma)\leq 1$, we obtain that
\begin{align}
    1{-}\gamma_B + \gamma_B G_{t+1}(\sigma) &\geq 1{-}\gamma + \gamma G_{t+1}(\sigma) \notag\\
    &\geq \gamma G_{t+1}(\sigma),
\end{align}
which along with \eqref{eq:return_recursive} concludes the proof of~\eqref{eq:l1_in1},~\eqref{eq:l1_in2}.
\end{proof}

Under a strategy pair $(\mu,\nu)$, it is straightforward to check the probability that a Rabin condition is satisfied in a game $\mathcal{G}$ (i.e., MC $\mathcal{G}_{\mu,\nu}$). All paths in the induced MC $\mathcal{G}_{\mu,\nu}$ eventually reach a BSCC 
$T \in \mathcal{B}(\mathcal{G}_{\mu,\nu})$ and visit its states infinitely many times. A path reaching a state in a BSCC  that does not contain any state in $B$ or $C$, does not satisfy the Rabin condition. We denote the set of all such states by $U_{\overline{BC}}$. Similarly, if a path reaches a state in a BSCC without any state in $C$ but with a state in $B$, it satisfies the Rabin condition; finally, if it reaches a state in a BSCC that does contain a state from  $C$, it does not satisfy the Rabin condition. We write $U_B$ and $U_C$ to denote the set of these states, respectively (formally defined in Lemma~\ref{lemma:U}). This reasoning reduces finding the probability of satisfying the Rabin condition to finding the probability of reaching a state in $U_B$, which allows us to focus on the reachability objective $\varphi_{U_B}\coloneqq\lozenge U_B$ instead of $\varphi_{B,C}$ defined in Theorem~\ref{theorem:rabin}.

We now show that the expected values of the returns~\eqref{eq:return_updated} (i.e., the state values) reflect the Rabin acceptance~condition. 

\begin{lemma} \label{lemma:U}
For any stochastic Rabin game $\mathcal{G}$ with $\textnormal{Acc}=\{(C,B)\}$ under a strategy pair $(\mu,\nu)$, it holds that:
\begin{align}\label{eq:L2i}
\lim_{\gamma \to 1^{-}} v_{\mu,\nu}^\gamma(s) &= 0 \text{ if }s \in U_{\overline{BC}},\\
\label{eq:L2ii}
\lim_{\gamma \to 1^{-}} v_{\mu,\nu}^\gamma(s) &= 1 \text{ if }s \in U_B,\\
\label{eq:L2iii}\lim_{\gamma \to 1^{-}} v_{\mu,\nu}^\gamma(s) &= 0 \text{ if }s \in U_C,
\end{align}
%
where the sets $U_{\overline{BC}}$, $U_B$ and $U_c$ are defined as:
\begin{equation}
\begin{split}
\label{eq:Us}
U_{\overline{BC}} &\coloneqq \{s_{\overline{BC}} \mid s_{\overline{BC}} {\in} T, T {\in} \mathcal{B}(\mathcal{G}_{\mu,\nu}), T{\cap}B {=} \varnothing,  T{\cap}C {=} \varnothing\}, \\
U_B &\coloneqq \{s_B \mid \exists T {\in} \mathcal{B}(\mathcal{G}_{\mu,\nu}), s_B {\in} T{\cap}B, T{\cap}C {=} \varnothing\},\\
U_C &\coloneqq \{s_C \mid \exists T {\in} \mathcal{B}(\mathcal{G}_{\mu,\nu}), s_C {\in} T{\cap}C\}. 
\end{split}
\end{equation}
\end{lemma}

\begin{proof}
To simplify our notation, in this proof and the proof of Theorem~\ref{theorem:rabin}, 
we use $v^\gamma$, $\Gamma$ and $Pr$ instead of $v^\gamma_{\mu,\nu}$, $\Gamma_{B,C}$ and $Pr^{\mathcal{G}}_{\mu,\nu}$. 
We define a return for a finite~path~as:
\begin{align}
    \vspace{-1em} G_{t:t+k}(\sigma) &\coloneqq \sum_{i=0}^{k} R_B(\sigma[t{+}i]) \prod_{j=0}^{i-1} \Gamma_{B,C}(\sigma[t{+}j])
\end{align} 

\noindent\textbf{Case I)}: 
Once a state $s_{\overline{BC}} \in U_{\overline{BC}}$ is reached, it is impossible to later visit a $B$ state and receive a nonzero reward;  thus, the values of all states in $U_{\overline{BC}}$ are zero and~\eqref{eq:L2i}~holds.

\vspace{0.5em}
\noindent\textbf{Case II)}: 
%
Once a state $s_B \in U_B$ is reached, it will be visited infinitely often as it belongs to a BSCC. Let $N_t$ be the time to the next visit to $s_B$ after visiting it at~$t$ -- i.e.,
\begin{align}
N_t = \min \{\tau \mid \sigma[t{+}\tau]=s_B,\tau>0 \}. 
\end{align}

From the definition of the return~\eqref{eq:return_updated}, it holds that
\begin{align}
    v^\gamma(s_B) &= 1{-}\gamma_B + \gamma_B\mathbb{E} [G_{t+1}(\sigma) \mid \sigma[t]{=}s_B] = \notag \\
    &= 1-\gamma_B+\gamma_B\mathbb{E}\Big[G_{t+1:t+N_t-1}(\sigma) \notag \\ &\hspace{-3em}+\Big(\prod\nolimits_{i=1}^{N_t-1} \Gamma({\sigma[t{+}i]})\Big)\cdot G_{t+N_t}(\sigma) \mid \sigma[t]{=}s_B\Big].
\end{align}
We can ignore the return of the prefix, $G_{t+1:t+N_t-1}(\sigma)$ and obtain an upper bound. Using $G_t(\sigma) \geq \gamma G_{t+1}(\sigma)$, we have
\begin{align}
    v^\gamma(s_B) &\geq 1{-}\gamma_B + \gamma_B\mathbb{E}\left[\gamma^{N_t-1} G_{t+N_t}(\sigma) \mid \sigma[t]{=}s_B\right]\notag
    \\ &\overset{\text{\small \ding{192}}}{\geq} 1{-}\gamma_B + \gamma_B\mathbb{E}\left[\gamma^{N_t-1} \mid \sigma[t]{=}s_B\right]v^\gamma(s_B)\notag\\
    &\overset{\text{\small \ding{193}}}{\geq}  1{-}\gamma_B + \gamma_B \gamma^{\mathbb{E}\left[N_t-1 \mid \sigma[t]{=}s_B\right]}v^\gamma(s_B)\notag\\
    &\geq 1-\gamma_B + \gamma_B \gamma^n v^\gamma(s_B) \label{eq:lb_1v}
\end{align}
where \ding{192} holds by the Markov property, \ding{193} holds from the Jensen's inequality, and $n\geq1$ is a constant.
From~\eqref{eq:lb_1v},
\begin{align}
    v^\gamma(s_B) &\geq \frac{1{-}\gamma_B}{1{-}\gamma_B \gamma^n} \overset{\text{\small \ding{194}}}{\geq} \frac{1{-}\gamma_B}{1{-}\gamma_B (1-n(1{-}\gamma))} = \notag \\
    &= \frac{1}{1 + n\frac{1{-}\gamma}{1{-}\gamma_B} - n(1{-}\gamma)}. \label{eq:lb_1} 
\end{align}
where 
\ding{194} holds as $(1- (1{-}\gamma))^n \geq 1 - n (1{-}\gamma)$ for $\gamma \in (0,1)$.
Finally, as $v^\gamma(s_B) \leq 1$ from Lemma~\ref{lemma:inequalities}, letting $\gamma, \gamma_B \to 1^-$ under the condition~\eqref{eq:gamma_lim}, 
concludes the proof of~\eqref{eq:L2ii}.

\vspace{0.5em}
\noindent\textbf{Case III)}: 
Similarly to the previous case, we define a stopping time for the number of time steps between two consecutive visits to a state $s_C \in U_C$ -- i.e., for a fixed~$t \in \mathbb{N}$
\begin{align}
M_t = \min \{\tau \mid \sigma[t{+}\tau]=s_C,\tau>0 \}. 
\end{align}
Now, we can split the value of $s_C$ into two expectations
\begin{align}
    v^\gamma(s_C) &= \gamma_C\mathbb{E}[G_{t+1:t+M_t-1}(\sigma) \mid \sigma[t]{=}s_C] + \notag \\ 
    &\hspace{-3em}\gamma_C\mathbb{E}\Bigg[\Bigg(\prod_{i=1}^{M_t-1} \Gamma({\sigma[t{+}i]})\Bigg)G_{t+M_t}(\sigma)
    \ \bigg\vert \
    \sigma[t]{=}s_C\Bigg]\hspace{-0.1em}.
\end{align}
Using the inequalities in Lemma~\ref{lemma:inequalities}, it holds that
\begin{align}
    v^\gamma(s_C) &\leq \gamma_C\mathbb{E} [1{-}\gamma_B^{M_t}] +\gamma_C \mathbb{E}[G_{t+M_t}(\sigma) \mid \sigma[t]{=}s_C] \notag \\
    &\overset{\text{\small \ding{192}}}{\leq} \gamma_C(1-\gamma_B^{m}) +\gamma_C \mathbb{E}[G_{t+M_t}(\sigma) \mid \sigma[t]{=}s_C] \notag \\
    &\leq
    1-\gamma_B^{m}+\gamma_C  v_{\mu,\nu}^\gamma(s_C) \label{eq:val_ineq}.
\end{align}
where \ding{192} holds by Jensen's inequality
and $m\geq0$ is a constant. Now, from \eqref{eq:val_ineq} it holds that
\begin{align}
    v^\gamma(s_C) \leq \frac{1{-}\gamma_B^{m}}{1{-}\gamma_C} \leq \frac{m(1{-}\gamma_B)}{1{-}\gamma_C}.
\end{align}
Thus, from~\eqref{eq:gamma_lim}, since $v^\gamma(s_C)$ is nonnegative,  \eqref{eq:L2iii} holds.
\end{proof}

We now provide the proof of Theorem~\ref{theorem:rabin}.

\begin{proof}[Proof of Theorem~\ref{theorem:rabin}]
We divide the expected return of a random path $\sigma$ visiting a state $s \in S$ depending on whether it satisfies $\varphi_{U_B}\coloneqq\lozenge U_B$ or not -- i.e.,
\begin{align}
    v^\gamma(s) = & \mathbb{E}[G_t(\sigma) \mid \sigma[t]{=}s, \sigma {\models} \lozenge U_B]Pr(\sigma{\models}\lozenge U_B) \notag \\
    + &\mathbb{E}[G_t(\sigma) \mid \sigma[t]{=}s, \sigma {\not\models} \lozenge U_B]Pr(\sigma {\not\models} \lozenge U_B) \label{eq:v_two_parts}
\end{align}
for some fixed $t \in \mathbb{N}$. 
Notice that $\sigma{\not\models}\lozenge U_B$ implies $\sigma[t{:}]{\not\models}\lozenge U_B$, and $\sigma{\models}\lozenge U_B$ implies $\sigma[t{:}] \models \lozenge U_B$ almost surely. Hence, $Pr(s{\models}\lozenge U_B)$ and $Pr(s{\not\models}\lozenge U_B)$ can be replaced with $Pr(\sigma{\models}\lozenge U_B)$ and $Pr(\sigma{\not\models}\lozenge U_B)$, respectively. 

After visiting the state $s$ at time $t$, let $L_t$ be the 
number of time steps until the first visit to a state in $U_B$ in~\eqref{eq:Us}  --~i.e.,
\begin{align}
L_t =\min \{\tau \mid \sigma[t{+}\tau]\in U_B,\tau>0 \}.  
\end{align}
Then, by Lemma \ref{lemma:inequalities}, it holds that
\begin{align}
    v^\gamma(s) &\geq \mathbb{E}[G_t(\sigma) \mid \sigma[t]{=}s, \sigma {\models} \lozenge U_B]Pr(s{\models}\lozenge U_B) \notag\\
    & \geq \mathbb{E}\left[\gamma^{L_t} G_{t+L_t}(\sigma) \mid \sigma[t]{=}s, \sigma {\models} \lozenge U_B\right]Pr(s{\models}\lozenge U_B) \notag\\
    & \overset{\text{\small \ding{192}}}{\geq}\mathbb{E}\left[\gamma^{L_t} \mid \sigma[t]{=}s, \sigma {\models} \lozenge U_B\right]\underline{v}^\gamma(U_B) Pr(s{\models}\lozenge U_B) \notag\\
    & \overset{\text{\small \ding{193}}}{\geq}\gamma^{\mathbb{E}\left[L_t \mid \sigma[t]{=}s, \sigma {\models} \lozenge U_B\right]}\underline{v}^\gamma(U_B) Pr(s{\models}\lozenge U_B) = \notag\\
    & = \gamma^l \underline{v}^\gamma(U_B)Pr(s{\models}\lozenge U_B) \label{eq:buchi_v_lower};
\end{align}
here, $\underline{v}^\gamma(U_B) = \min_{s_B\in U_B} v^\gamma(s_B)$, $l$ is constant, and \ding{192} and \ding{193} hold from the Markov property and Jensen's~inequality.

Similarly, after leaving $s$ at $t$, let $L_t'$ be the 
 number of time steps until the first time a state in $U_{\overline{BC}} \cup U_C$ is reached~--~i.e.,
\begin{align}
    L_t' & =\min \big\{\tau \mid \sigma[t{+}\tau]\in U_{\overline{BC}} \cup U_C, \tau>0 \big\}.
\end{align}
 Then, using Lemma~\ref{lemma:inequalities} and the Markov~property, it holds that
\begin{align}
    v^\gamma(s) &\leq \mathbb{E}[G_t(\sigma) \mid \sigma[t]{=}s, \sigma {\not\models} \lozenge U_B]Pr(s {\not\models} \lozenge U_B)\notag \\
    &\hspace{2em}+Pr(s{\models}\lozenge U_B) \notag\\
    &\leq \mathbb{E}[1{-}\gamma_B^{L_t'}\mid \sigma[t]{=}s, \sigma {\not\models} \lozenge U_B]Pr(s {\not\models} \lozenge U_B)\notag \\
    &\hspace{2em}+Pr(s{\models}\lozenge U_B) \notag\\
    &\leq 1{-}\gamma_B^{\mathbb{E}[L_t' \mid \sigma[t]{=}s, \sigma \not\models \lozenge U_B]} Pr(s {\not\models} \lozenge U_B)\notag \\
    &\hspace{2em}+Pr(s{\models}\lozenge U_B) \notag\\
    &= (1{-}\gamma_B^{l'}) Pr(s {\not\models} \lozenge U_B)+Pr(s{\models}\lozenge U_B),
    \label{eq:b2}
\end{align}
where $l'$ is some constant. The upper bound~\eqref{eq:b2} and the lower bound~\eqref{eq:buchi_v_lower} (due to~\eqref{eq:L2ii}) approach the probability $Pr(s{\models}\lozenge U_B)$
as $\gamma \to 1^{-}$, thereby concluding the proof. 
\end{proof}

\subsection{Reduction to Stochastic Rabin($1$) Games} \label{section:rabin_k}

We now provide a generalization of our approach from Section~\ref{section:rabin_1} to the Rabin conditions with $k$ pairs.~The~idea is to construct $k$ different stochastic Rabin($1$) games and connect them with $\varepsilon$ actions so that the controller is able to switch between the Rabin pairs it aims to satisfy. 

\begin{definition}[$k$-copy Game]
Let $[n]$ denote the set $\{1,2,\dots,n\}$ for a positive integer $n$. For a given stochastic Rabin($k$) game $\mathcal{G}=(S, (S_\mu, \allowbreak S_\nu), A, P, s_0, \textnormal{Acc})$, with $\textnormal{Acc} {=} \{(C_i,B_i))_{i=0}^k\}$, a $k$-copy game $\mathcal{G}^\star=(S^\star,(S^\star_\mu,S^\star_\nu),\allowbreak A^\star,P^\star,s^\star_0,\textnormal{Acc}^\star)$ is a stochastic Rabin($1$) game defined by:
\begin{itemize}
    \item $S^\star = \left(S_\mu \times [2k]\right) \cup \left(S_\nu \times [k]\right)$ is the augmented state set with $S^\star_\mu = S_\mu \times [k]$ the controller and
    $S^\star_\nu = S \setminus S^\star_\mu$ the environment states, and $s^\star_0=\langle s_0,1\rangle$ is the initial~state;
    \item $A^\star = A \cup \{\varepsilon_i \mid i \in [k]\} \cup \{\varepsilon'\}$ is the set of actions; 
    \item $P^\star:S^\star {\times} A^\star {\times} S^\star \to [0,1]$ is the transition function defined as $P^\star(\langle s,i \rangle,a,\langle s',i'\rangle)$
    \begin{align*}
        =\begin{cases}
            P(s,a,s') & \textrm{if }a \in A, i=i', \\
            1 & \textrm{if }s\in S_\mu, s=s', a=\varepsilon_i, i'=k+i, \\
            1 & \textrm{if }s\in S_\nu, s=s', a=\varepsilon', i'=i-k, \\
            0, & \textnormal{otherwise};
        \end{cases} \notag 
    \end{align*}
    \item $\textnormal{Acc}^\star=\{(C^\star,B^\star)\}$ is the Rabin accepting set where
    \begin{align*}
    \hspace{-0.5em}C^\star &\coloneqq \{\langle s,i \rangle \mid s\in C_i, i \in [k]\text{ or } s\in S_\mu, i \in [2k]\setminus[k]\}, \\
    \hspace{-0.5em}B^\star &\coloneqq \{\langle s,i \rangle \mid s\in B_i, \ i \in [k]\}.
    \end{align*}
\end{itemize}
\end{definition}

\begin{figure}
    \centering
    \includegraphics{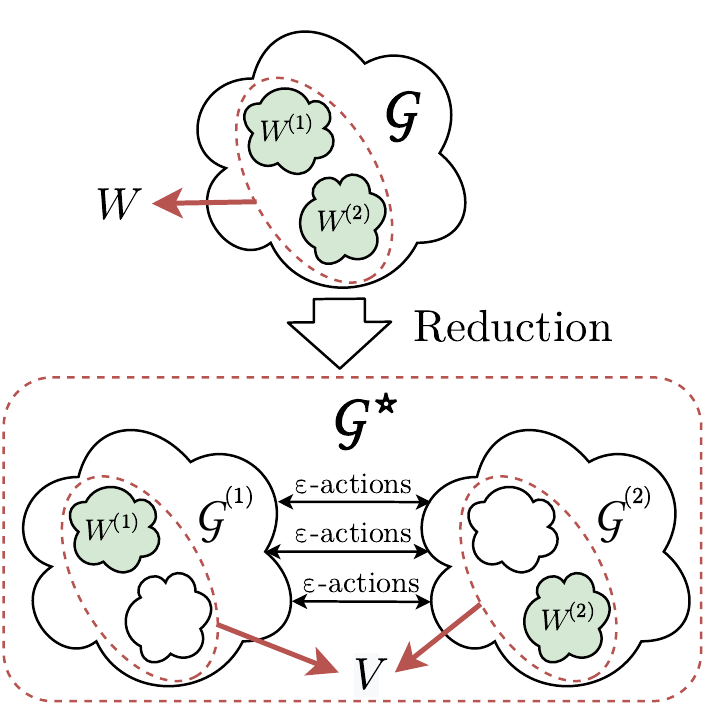}
    \caption{A $k$-copy game obtained from a stochastic Rabin($2$) game. $W^{(1)}$ and $W^{(2)}$ denote the winning sets for the first and the second Rabin pair respectively.}
    \label{fig:k-copy}
\end{figure}

Intuitively, the $k$-copy game $\mathcal{G}^\star$ consists of $k$ exact copies of the original game $G$ for each accepting pair, and a dummy state $\langle s,i{+}k \rangle$ for every controller state $s \in S_\mu$ for each copy $i \in [k]$ (Fig.~\ref{fig:k-copy}). The controller can choose an $\varepsilon_j$ in a state $\langle s, i \rangle$ and makes a transition to the dummy environment state $\langle s,j{+}k \rangle$ where the environment can only take the action $\varepsilon'$, which makes a transition to the controller state $\langle s,j \rangle$. The idea here is to connect the $k$ copies of the original game using these $\varepsilon$-actions so that in any state, the controller can jump to the $j$-th copy via an $\varepsilon_j{\to}\text{a-dummy-state}{\to}\varepsilon'$ sequence. All the dummy states belong to $C^\star$, prohibiting the $\varepsilon$-actions from being visited infinitely many times. Also, the only states belonging to $C^\star$ and $B^\star$ in the $i$-th copy are the ones belonging to $C_i$ and $B_i$, respectively. This allows each accepting pair to be independently satisfied in its corresponding copy as stated in the following theorem.

\begin{theorem}\label{theorem:star_probs}
Let $\mathcal{G}^{(j)}$ be the stochastic Rabin($1$) game obtained from a Rabin($k$) game $\mathcal{G}$ by replacing $\textnormal{Acc}$ with $\{(C_j,B_j)\}$, and $W^{(j)}$ be the set of winning states such that for any $s\in W^{(j)}$, $Pr_*^{\mathcal{G}^{(j)}}(s \models \varphi_{B_j,C_j})=1$. Then, for any $\langle s,i \rangle \in S^\star$, it holds that
\vspace{-2pt}
\begin{align} \label{eq:lemma1}
    Pr^{\mathcal{G}^\star}_*(\langle s,i \rangle \models \varphi_{B^\star,C^\star}) = Pr^{\mathcal{G}^\star}_*(\langle s,i \rangle \models \lozenge V),
\end{align}
where $V = \big\{\langle s',i' \rangle \in S^\star \mid s' \in \bigcup\nolimits_{j=0}^{k} W^{(j)}\big\}$.
\end{theorem}
\begin{proof}
We prove \eqref{eq:lemma1} in two directions.

\vspace{2pt}\noindent$\geq:$ If a state $\langle s,i \rangle \in V$, then, by definition, there must be some $j$ such that $s \in W^{(j)}$. The controller can make a transition from $\langle s,i \rangle$ to $\langle s,j \rangle$ via the $\varepsilon$-actions and satisfy $\varphi_{B^\star,C^\star}$ by satisfying $\varphi_{B_j,C_j}$. Thus, the control strategy that maximizes the reachability probabilities in the worst case also guarantees that the satisfaction probabilities are at least the maximin reachability probabilities, i.e., $Pr^{\mathcal{G}^\star}_*(\langle s,i \rangle \models \varphi_{B^\star,C^\star}) \geq Pr^{\mathcal{G}^\star}_*(\langle s,i \rangle \models \lozenge V)$.

\vspace{4pt}\noindent$\leq:$ All the transitions via the $\varepsilon$-actions pass through a state in $C^\star$. Under any strategy pair, the BSCCs having $\varepsilon$-transitions of the induced MC are rejecting. Since without some $\varepsilon$-transitions, it is not possible for a BSCC to contain states from two different accepting pairs, an accepting BSCC must satisfy only a single pair. In addition, in the worst case, the satisfaction probability  can be maximized by maximizing the probability of reaching a state that belongs to an accepting BSCC for any environment strategy. Thus, such states must be a winning state for some accepting pair, which implies that $Pr^{\mathcal{G}^\star}_*(\langle s,i \rangle \models \varphi_{B^\star,C^\star}) \leq Pr^{\mathcal{G}^\star}_*(\langle s,i \rangle \models \lozenge V)$.
\end{proof}

Any control strategy $\mu^\star$ in $\mathcal{G}^\star$ has a corresponding finite-memory strategy $\mu$ in the Rabin($k$) game $\mathcal{G}$, which can be captured by a deterministic finite automaton (DFA) with $k$ states. In state $s$, the state of the DFA changes from state $i{\in}[k]$ to $j{\in}[k]$, 
if $\mu^\star(\langle s,i \rangle) {=} \varepsilon_j$; the DFA state stays the same and the control strategy $\mu$ chooses action $a{\in} A$ if  $\mu^\star(\langle s,i \rangle) {=} a$. If $\mu^\star$ is a maximin strategy for $\mathcal{G}^\star$~then under the induced strategy $\mu$, the controller satisfies~the~acceptance condition with probability that is not lower than the probability of reaching a winning state of an accepting pair.

\begin{corollary} \label{corollary:acceptance_lower_bound}
A maximin control strategy for $\mathcal{G}^\star$ of a stochastic Rabin game $\mathcal{G}$ with $k$ accepting pairs induces a control strategy $\mu$ for $\mathcal{G}$ such that
\begin{align}
    Pr_{\mu,\nu}\left(\mathcal{G} \models \varphi_{\textnormal{Acc}} \right) \geq 
    Pr_*\left(\mathcal{G} \models \lozenge W \right) \label{eq:acceptance_lower_bound}
\end{align}
for any environment strategy $\nu$, where 
$$\varphi_{\textnormal{Acc}} \coloneqq  \bigvee\limits_{(B_i,C_i)\in\textnormal{Acc}} \left({\square} {\lozenge} B_i {\wedge} {\neg} {\square} {\lozenge} C_i\right), \quad W \coloneqq \bigcup\limits_{i=1}^k W^{(i)}$$ 
with $W^{(i)}$ defined as in Theorem~\ref{theorem:star_probs}.
\end{corollary}
\begin{proof}
For any environment strategy $\nu$ in $\mathcal{G}$ we can construct a corresponding environment strategy $\nu^\star$ in $\mathcal{G}^\star$ such that $\nu^\star(\langle s,i \rangle)=\nu(s)$ for all $i \in [k]$ and $\nu^\star(\langle s,i \rangle)=\varepsilon'$ for all $[2k]\setminus[k]$. Note that the strategy pairs $(\mu,\nu)$ and $(\mu^\star,\nu^\star)$ induce the same MCs. Since satisfying $(C_j,B_j)$ satisfies $\varphi_{\textnormal{Acc}}$, we have $Pr_{\mu,\nu}\left(\mathcal{G} \models \varphi_{\textnormal{Acc}} \right) \geq Pr_{\mu^\star,\nu^\star}\left(\mathcal{G}^\star \models \varphi_{B^\star,C^\star} \right)$, which combined with Theorem \ref{theorem:star_probs} concludes the proof.
\end{proof}

\begin{figure}
    \centering
    \includegraphics[width=0.42\textwidth]{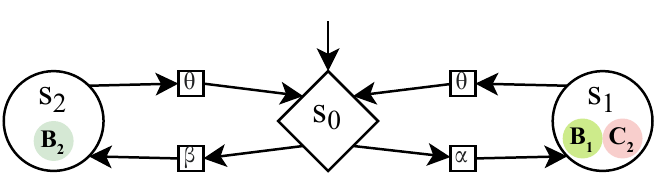}
    \caption{A stochastic Rabin game with two accepting pairs: $(C_1,B_1){=}(\varnothing,\{s_1\})$ and $(C_2,B_2){=}(\{s_1\},\{s_2\})$. The initial state $s_0$ is controlled by the environment; $s_1$ and $s_2$ are the controller states; $\alpha$, $\beta$ and $\theta$ are actions; and all transition probabilities are 1.}
    \label{fig:srg}
\end{figure}

The induced control strategy $\mu$ guarantees a satisfaction probability that is larger than or equal to $Pr_*\left(\mathcal{G} \models \lozenge W \right)$. However, for some stochastic games, there exist some control strategies with improved lower bounds. This is due to the fact that there could be other winning states not belonging to $W$. One such game is illustrated in Fig.~\ref{fig:srg}. In this game,~if the environment chooses the action $\beta$ all the time, the generated path  does not satisfy the first accepting pair; similarly, if it chooses $\alpha$, the path does not satisfy the second accepting~pair. Thus, neither $s_1$ nor $s_2$ belong to $W$ of~$\mathcal{G}$; yet, they are winning states since independently from the used environment strategy, an accepting pair is always satisfied.

\subsection{Controller Synthesis via Reinforcement Learning}

We now state the main result of our approach.
\begin{theorem} \label{theorem:main_result}
For a given stochastic Rabin($k$) game $\mathcal{G}$, there exists a $\gamma'$ such that for all $\gamma$, such that $\gamma'<\gamma<1$, a minimax-Q learning algorithm using the reward and the discount functions in Theorem~\ref{theorem:rabin} is guaranteed to converge to a strategy $\mu$ satisfying (\ref{eq:acceptance_lower_bound}). 
\end{theorem}

\begin{proof}
The claim directly  follows from Theorem~\ref{theorem:rabin}, Corollary~\ref{corollary:acceptance_lower_bound}, and the fact that pure and memoryless strategies are finite and sufficient for both the controller and the environment in stochastic Rabin($1$) games~\cite{chatterjee2012}.
\end{proof}

For a given stochastic game and an LTL specification, we can reduce the control synthesis problem to finding a maximin controller strategy in a stochastic Rabin($k$) game~$\mathcal{G}$ using the standard automata-based approach described in Section \ref{section:product}. This can be further reduced to finding a strategy that maximizes the probability of satisfying a single Rabin pair in the worst case, in a stochastic Rabin($1$) game $\mathcal{G}^\star$ using the method from Section \ref{section:rabin_k}. In Section \ref{section:rabin_1}, we provide an approach that transforms the objective of satisfying of a Rabin pair to a discounted reward maximization objective, allowing the use of RL to synthesize controllers. 

Algorithm \ref{alg:alg1} summarizes the steps of our approach. Here, $\alpha$ is the learning rate and the bold \textbf{s} character denotes a three-dimensional state vector: the state of the original game, the DRA state, and the index of Rabin pair. After the construction of $\mathcal{G}^\star$, the algorithm performs simple minimax-Q learning. In each iteration of learning, the algorithm derives an $\epsilon$-greedy strategy pair, which means that under these strategies, the controller and the environment randomly choose their actions with probability of $\epsilon$, as well choose their best action with probability of $1{-}\epsilon$. After the convergence, the algorithm returns a maximin control strategy $\mu^\star_*$ for $\mathcal{G}^\star$, which induces a finite-memory strategy for the original game, which guarantees the lower bound provided in Corollary~\ref{theorem:main_result}.

Computing the winning states in stochastic Rabin games is NP-Complete in the number accepting pairs \cite{chatterjee2012}. Thus, it is unlikely to construct a stochastic Rabin($1$) game from any given stochastic Rabin($k$) game without an exponential blowup in the number of states. Instead, we provide a simple yet powerful approach combining all the accepting pairs that only requires a linear increase in the number of states. 

\begin{algorithm}
\begin{algorithmic}
\STATE \textbf{Input:} LTL formula $\varphi$, stochastic game $\mathcal{G}$
\STATE Translate $\varphi$ to a DRA $\mathcal{A}_\varphi$
\STATE Construct the product $\mathcal{G}^\times$ of $\mathcal{G}$ and $\mathcal{A}_\varphi$
\STATE Reduce  $\mathcal{G}^\times$ to  $\mathcal{G}^\star$
\STATE Initialize $Q(\textbf{s}, a)$ on $\mathcal{G}^\star$
\FOR {$t = 0,1,\dots, T$}
\STATE Derive an $\epsilon$-greedy strategy pair $(\mu^\star,\nu^\star)$ from $Q$
\STATE Take the action $a_t \leftarrow \begin{cases}\mu^\star(\textbf{s}_t), & \textbf{s}_t {\in} S^\star_\mu \\ \nu^\star(\textbf{s}_t), & \textbf{s}_t {\in} S^\star_\nu \end{cases}$
\STATE Observe the next state $\textbf{s}_{t+1}$
\STATE $Q(\textbf{s}_t,a_t) \leftarrow (1-\alpha) Q(\textbf{s}_t,a_t) + \alpha  R(\textbf{s}_t)$
\STATE \hspace{0.1em} $+ \alpha  \Gamma(\textbf{s}_t) \cdot \begin{cases}\max_{a'} Q(\textbf{s}_{t+1},a'), & \textbf{s}_{t+1} {\in} S^\star_\mu \\ \min_{a'} Q(\textbf{s}_{t+1},a'), & \textbf{s}_{t+1} {\in} S^\star_\nu \end{cases}$
\ENDFOR
\STATE Get a greedy control strategy $\mu^\star_*$ from $Q$
\RETURN $\mu^\star_*$
\end{algorithmic}
\caption{Model-free RL for control synthesis in stochastic games from LTL specifications.}\label{alg:alg1}
\end{algorithm}


\section{Experimental Results}

We implemented a software tool \cite{csrl2020} in \textit{Python} that takes a description of a labeled stochastic game and an LTL specification of a task, and outputs the desired control strategy using RL. Our tool uses \textit{Rabinizer 4} \cite{kvretinsky2018} to translate the given LTL formula into a DRA then constructs the product game of the given game and the DRA. It performs minimax-Q learning using the presented reward and discount~functions.

During learning, $\epsilon$-greedy strategies are followed by both players after starting in a random state, and the episodes are terminated after 1K steps. We set the parameter $\epsilon$ and the learning rate $\alpha$ to 0.5 and gradually decreased them to $0.05$ during learning; we used the discount factors of $\gamma_C{=}1{-}(0.01)$, $\gamma_B{=}1{-}(0.01)^2$ and $\gamma{=}1{-}(0.01)^3$. 

We 
considered robot planning tasks on two different scenarios in two-dimensional $(5\times 5)$ grid worlds. The controller navigates a robot which occupies one cell and can move to adjacent four cells in a single time step using four actions: \textit{North, South, East} and \textit{West}. There are three types of cells: empty cells, cells with an obstacle and absorbing cells. When the robot tries to move to a cell with an obstacle, it hits the obstacle and stays in its previous position; once it moves to an absorbing cell it cannot leave it. 
In the figures, obstacles and absorbing cells are represented by filled and empty circles. Each cell is labeled with a set of atomic propositions, depicted as encircled letters in the lower part of the cells.

\subsection{Robust Control Design}
In our first case study,
the robot can unpredictably move in a direction that is orthogonal to the intended direction after taking an action. We model this source of nondeterminism as the environment,  observing the actions of the controller and acting to minimize the probability that the controller achieves the given task. The controller, in this case, tries to come up with a robust and conservative strategy that maximizes the probability of achieving the task in the worst case.

\begin{figure}
    \centering
    \includegraphics[width=0.3\textwidth]{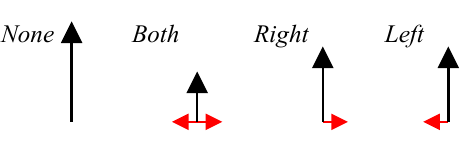}
    \caption{Illustration of the actions of the environment for the control action \textit{North}. The arrow lengths are loosely proportional to the probabilities of the movement directions.}
    \label{fig:secondary_actions}
\end{figure}
    
The environment can reactively choose one of the following four actions: \textit{None, Both, Right} and \textit{Left} (Fig.~\ref{fig:secondary_actions}). Fpr \textit{None}, the robot moves in the intended direction without any disturbance. If \textit{Both} is chosen, it can go sideways with a probability of 0.2 (i.e., 0.1 for each direction). If  \textit{Right}(\textit{Left}) is chosen,  the robot moves as intended with probability of 0.9 and in right(left) direction with probability of 0.1. 

The objective is to visit a state labeled with $b$ and a state labeled with $c$ infinitely often, and the safe states, labeled with either $d$ or $e$, should not be left after a certain point of time. The task is formally described by the LTL formula
\begin{align}
    \varphi_1 = \square \lozenge b \wedge \square \lozenge c \wedge (\lozenge \square d \vee \lozenge \square e), \label{eq:robust_controller}
\end{align}
which we translated into a DRA with two accepting pairs.

Fig.~\ref{fig:robust_controller} depicts the grid world we used and the strategy obtained for it after 128K episodes. The objective cannot be satisfied by visiting the states (labeled with $b$ and $c$) at the top-left or the top-right corner of the grid because the environment can force the robot to leave the safe states. The only possible way to achieve the task is going from the state labeled with $b$ to the state $c$ and vice versa without leaving~the~safe states. The strategy in Fig.~\ref{fig:robust_controller} ensures that the robot stays below the second row once it reaches the area, and does not visit the unsafe state in the middle regardless of the environment's actions. Also, under the strategies in Fig.~\ref{fig:robust_controller_b_to_c},~\ref{fig:robust_controller_c_to_b}, the robot eventually reaches the states labeled with $b$ and $c$,~respectively.

\begin{figure}
    \begin{subfigure}{0.24\textwidth}
        \centering
        \includegraphics[width=\textwidth]{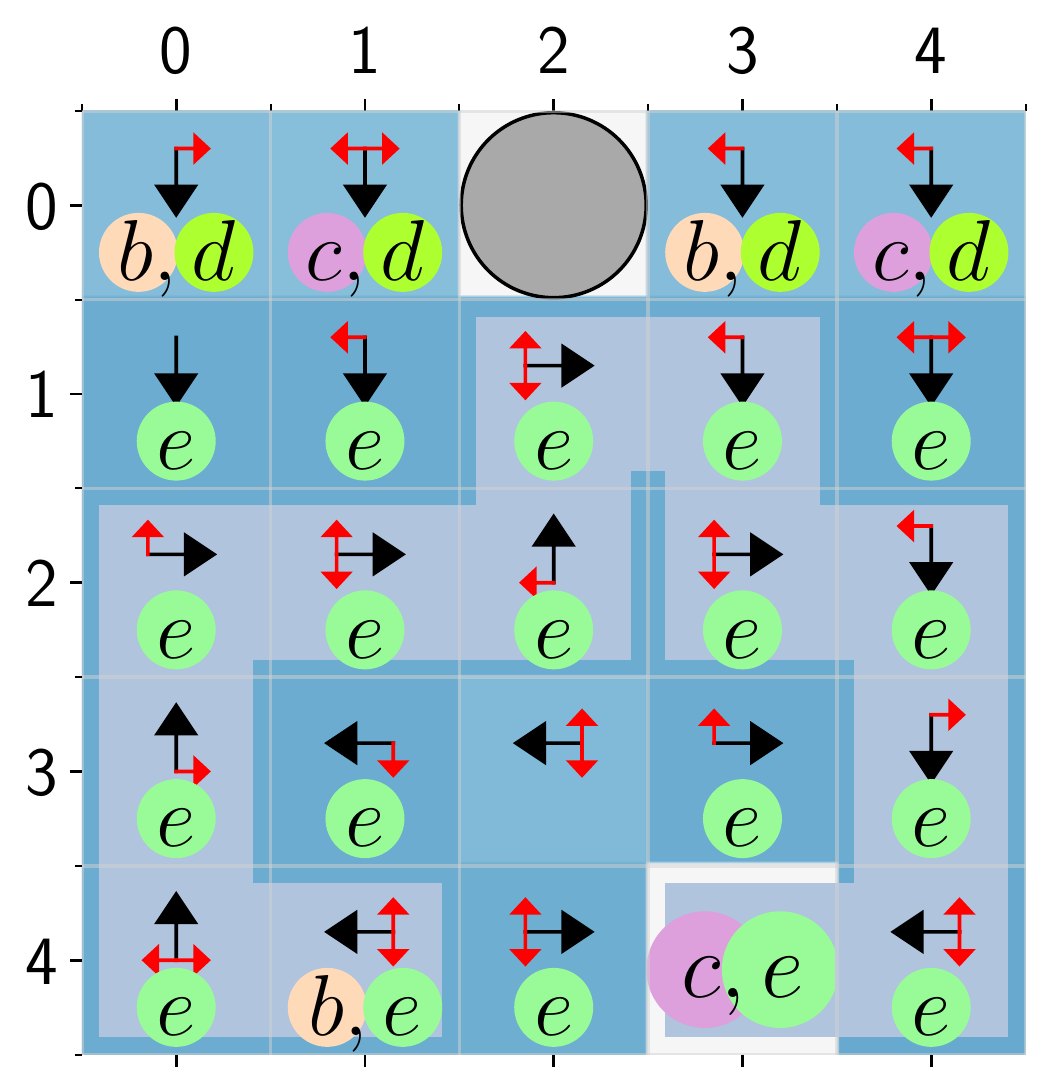}
        \caption{Strategy from $b$ to $c$}
        \label{fig:robust_controller_b_to_c}
    \end{subfigure}
    \begin{subfigure}{0.24\textwidth}
        \centering
        \includegraphics[width=\textwidth]{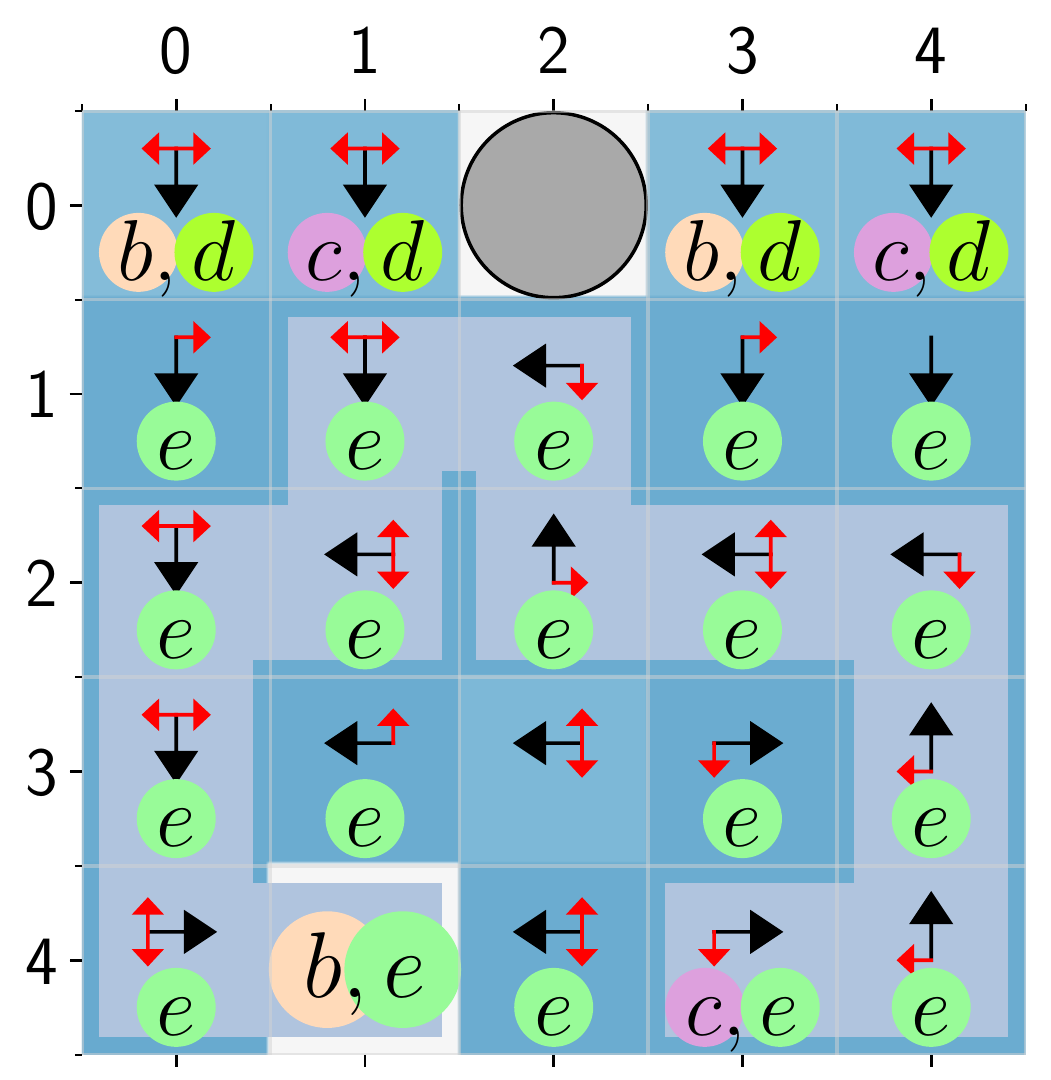}
        \caption{Strategy from $c$ to $b$}
        \label{fig:robust_controller_c_to_b}
    \end{subfigure}
    \caption{The objective strategy and the estimated maximal probabilities of satisfying  $\varphi_1$ from~\eqref{eq:robust_controller}. The most likely path is highlighted in a lighter blue.
    } 
    \label{fig:robust_controller}
\end{figure}

\subsection{Avoiding Adversary}
In this case study, the robot movement is not affected by the environment actions. Instead, we model the imprecise movement by a fixed probability distribution --  the robot moves as intended with probability of 0.8 and goes to the right or the left side of the intended direction with probabilities of 0.1. Another agent is controlled by an adversary (the environment), which can take the same four actions as the controller, 
with the same probability~distribution. 

The size of the state space here is $(5\times5)\times(5\times5)=625$ since there are two independent agents. The labeling function is based on the position of the first agent as
\begin{align}
    L(\langle s_1,s_2 \rangle) \coloneqq \begin{cases}
        L(s_1), & s_1 \neq s_2 \\
        L(s_1) \cup \{a\}, & s_1 = s_2
    \end{cases}
\end{align}
The label $a$ represents the state where both agents are in the same position (adversary 'catches' the robot). The robot objective is the same as in the first scenario~\eqref{eq:robust_controller}, except that it additionally needs to avoid the adversary at all costs -- i.e.,
\begin{align}
    \varphi_2 = \varphi_1 \wedge \square \neg a. \label{eq:adversary}
\end{align}

Fig.~\ref{fig:adversary} shows the control strategy obtained after 512K episodes. There are four safe zones in this grid world: one at the top-left, another at the top-right, and two at the bottom part of the grid. The robot or the adversary can get trapped in a sink state with probability $p\geq0.2$ while traveling between the top and the bottom parts of the grid. Thus, the optimal strategy for the controller is not to switch zones unless the adversary is in the same zone. For example, in Fig.~\ref{fig:adversary_0_0}, if the robot is in the bottom part, the controller should not try to move the robot to the top-right part, a farther safe zone, because there is a chance ($p\geq0.2$) that the adversary ends up with a sink state if she tries to move to the bottom part. If the robot is in the top-right part, the controller should switch to the second Rabin pair via $\varepsilon_2$ and make the robot stay in the same zone. However, in Fig.~\ref{fig:adversary_3_1}, the robot cannot stay in the bottom part because otherwise the adversary will eventually catch~her.

\begin{figure}
    \begin{subfigure}{0.24\textwidth}
        \centering
        \includegraphics[width=\textwidth]{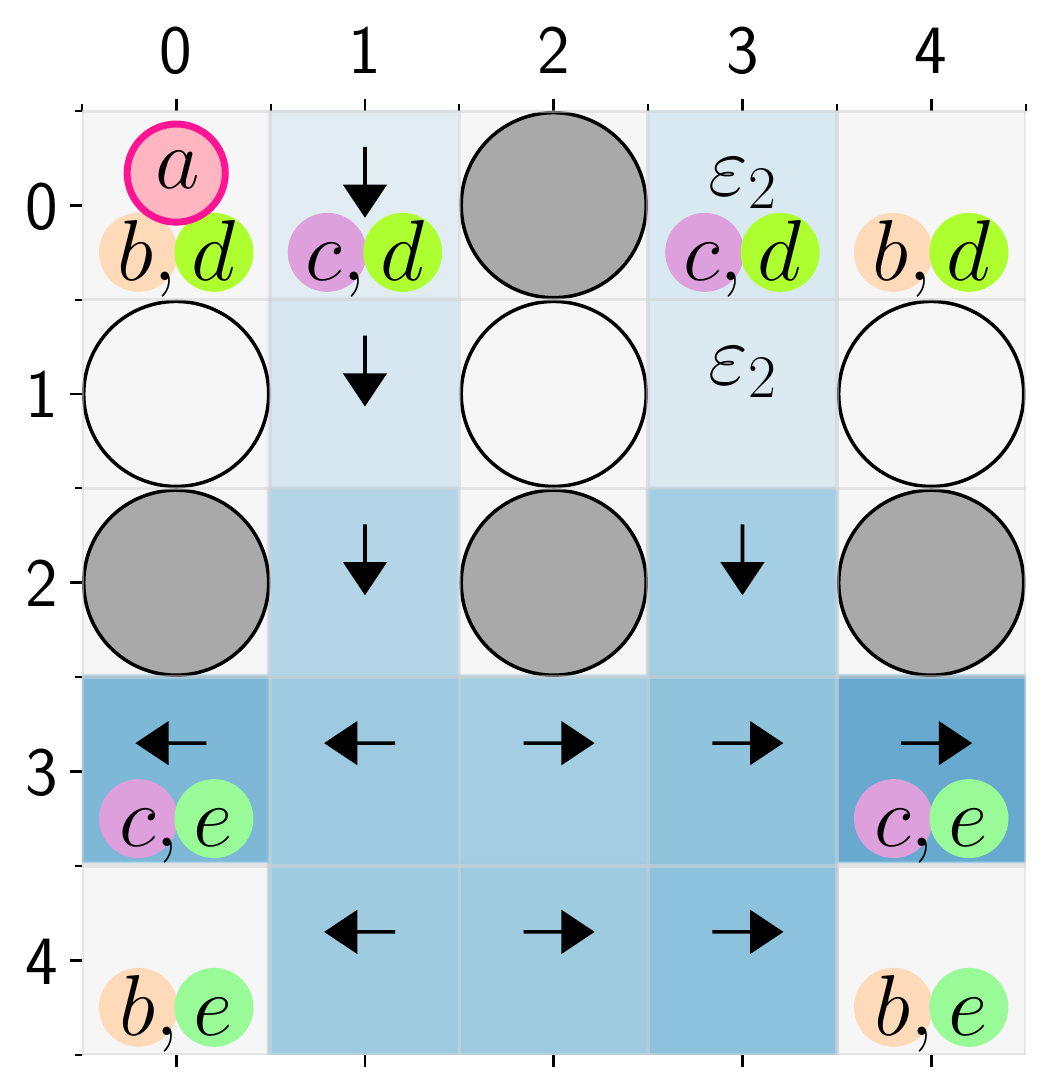}
        \caption{\footnotesize Adversary is at $(0,0)$ and $i{=}1$}
        \label{fig:adversary_0_0}
    \end{subfigure}
    \begin{subfigure}{0.24\textwidth}
        \centering
        \includegraphics[width=\textwidth]{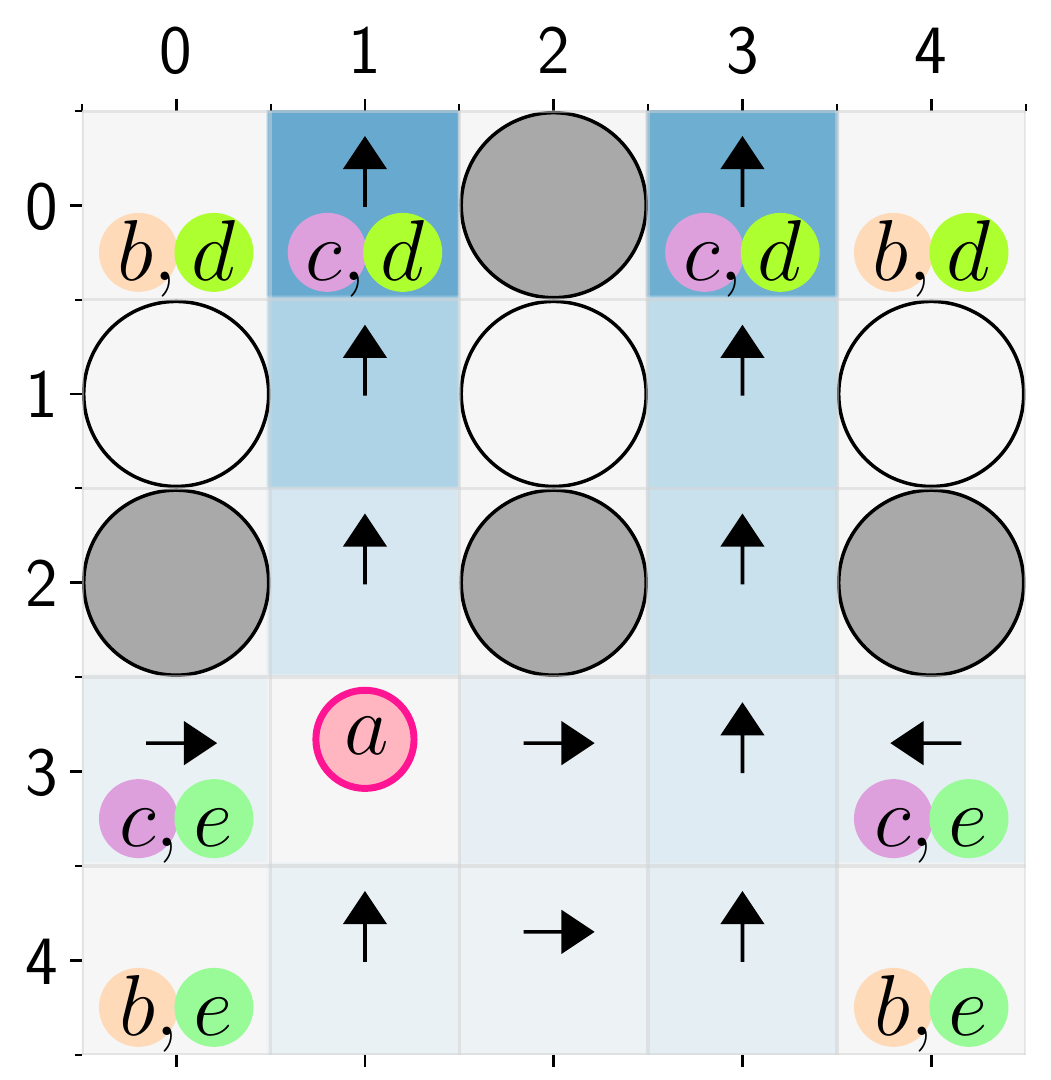}
        \caption{\footnotesize Adversary is at $(3,1)$ and $i{=}2$}
        \label{fig:adversary_3_1}
    \end{subfigure}
    \caption{The control strategies obtained for $\varphi_2$ from~\eqref{eq:adversary}. The values of the states are represented by the shades of blue (the darker, the higher value), which are the estimation of how likely the controller satisfies the objective.} 
    \label{fig:adversary}
\end{figure}

\section{Conclusions}

In this paper, we introduced an RL-based approach for synthesis of controllers from LTL specifications 
in stochastic games. We first provided a reduction from this synthesis problem to the problem of finding a control strategy in a stochastic Rabin game with a single accepting pair. We introduced a rewarding and discounting mechanism that transforms the objective of maximizing the (minimal/worst-case) probability of satisfying the Rabin condition into the objective of maximizing the discounted reward, and introduced an RL algorithm to find such a policy. We then provided a reduction allowing us to generalize our approach to any LTL specification, with the Rabin condition having $k>1$ accepting pairs, with a lower bound on the satisfaction probabilities. Finally, we  showed the applicability of our approach on two case path planning case-studies.


\bibliographystyle{unsrt}
\bibliography{references,yu}

\end{document}